\documentclass[11pt,a4paper]{article}

\usepackage[utf8]{inputenc}
\usepackage[T1]{fontenc}
\usepackage{lmodern}
\usepackage{amsmath,amssymb,amsthm}
\usepackage{mathtools}
\usepackage{booktabs}
\usepackage{hyperref}
\usepackage{cleveref}
\usepackage{natbib}
\usepackage[margin=1in]{geometry}
\usepackage{microtype}
\usepackage{xcolor}
\usepackage{enumitem}
\usepackage{graphicx}

\newtheorem{theorem}{Theorem}
\newtheorem{corollary}[theorem]{Corollary}
\theoremstyle{definition}
\newtheorem{definition}[theorem]{Definition}
\theoremstyle{remark}

\hypersetup{colorlinks=true, linkcolor=blue!60!black, citecolor=green!50!black, urlcolor=blue!70!black}
\hypersetup{
  pdftitle={The Library Theorem: How External Organization Governs Agentic Reasoning Capacity},
  pdfauthor={Zachary Mainen},
  pdfsubject={AI, computational complexity, agentic reasoning},
  pdfkeywords={library theorem, indexed retrieval, agentic reasoning, I/O complexity, chain-of-thought}
}

\newcommand{\bigO}{\mathcal{O}}

\title{%
  \textbf{The Library Theorem:}\\[4pt]
  \textbf{How External Organization Governs Agentic Reasoning Capacity}%
}

\author{%
  Zachary~F.~Mainen\thanks{Corresponding author. \texttt{zmainen@neuro.fchampalimaud.org}}\\
  Champalimaud Foundation, Lisbon, Portugal%
}

\date{March 2026}

\begin{document}
\maketitle

\begin{abstract}
Writing transformed human cognition by providing a capacious workspace for extended chains of thought, making reasoning indelible and widely communicable.
Then, by the invention of the index, writing became a retrieval system, extending its reach even further.
The library and the search engine are canonical demonstrations of index power.
Externalized reasoning is already exploited by transformer-based agents through chain-of-thought.
But structured retrieval---indexing over one's own reasoning state---remains underexplored.

We formalize the transformer context window as an I/O page and prove that tool-augmented agents with indexed external memory achieve exponentially lower retrieval cost than agents restricted to sequential scanning: $\bigO(\log_b N)$ versus $\Omega(N)$ page reads per query, and $\bigO(T \log_b T)$ versus $\Theta(T^2)$ cumulative cost over $T$ reasoning steps.
We test these predictions on a controlled lookup benchmark across three content types---random hashes, ordered integers, and encyclopedia entries---varying store size from 50 to 5,000 items, and replicate key conditions across two model generations (GPT-4o-mini and GPT-5.4\footnote{GPT-5.4 was accessed via OpenRouter (\texttt{openai/gpt-5.4}).}).
On abstract content (hashes and integers), the indexed agent achieves median 1 page read regardless of store size, confirming the $\bigO(1)$ prediction.
Sorted pages without an index---which in principle enable binary search---fail to close the gap: the weaker model cannot sustain binary search at scale, and the stronger model achieves near-optimal $\log_2 N$ search but still loses to the index by $5\times$.
On familiar content (encyclopedia entries), a competing failure mode emerges: the model recognizes the domain, bypasses the retrieval protocol, and generates answers from parametric memory, producing catastrophic token expenditure even when the index structure itself is sound.
This \emph{parametric memory competition} dissociates the two cognitive operations that indexing combines: understanding content (where language models excel) and following navigational protocols (where they fail when understanding tempts them to shortcut).
The result argues for a separation of concerns: use language models for index \emph{construction}, where semantic understanding helps, and deterministic algorithms for index \emph{traversal}, where it hurts.
\end{abstract}

\section{Introduction}\label{sec:intro}

Transformer-based agents extend the reach of single-inference language models by issuing tool calls across multiple steps, accumulating context and acting on intermediate results.
Chain-of-thought prompting \citep{yao2023react} demonstrated that reasoning quality improves when intermediate steps are externalized as text.
Tool-augmented agents go further: they read files, query databases, execute code, and write back results, constructing an extended reasoning trace that spans many inference calls.

A basic question about this architecture has not been answered formally: \emph{how does the organization of the agent's external state affect its reasoning cost?}
The agent's conversation history is a sequential tape---each new message is appended, and retrieval requires scanning.
An indexed file system provides random access by name.
The difference is the difference between a pile of notes and a library.

We formalize this difference.
The context window is an I/O page: a bounded buffer through which the model exchanges information with an external store.
Each inference step is a page operation.
The store may be organized sequentially (flat conversation history) or hierarchically (indexed file system, B-tree, database).
Under sequential access, finding a specific item in a store of $N$ pages costs $\Omega(N)$ operations.
Under indexed access, the same operation costs $\bigO(\log_b N)$ where $b$ is the branching factor.
The separation is exponential, and it compounds: over $T$ reasoning steps that read from an accumulating store, the costs are $\Theta(T^2)$ versus $\bigO(T \log_b T)$.

These are not exotic scenarios.
Every agent that maintains a conversation history and retrieves from it is in the sequential regime.
Every agent that writes intermediate results to named files and reads them back by name is in the indexed regime.
The theorem quantifies what practitioners already suspect: that structured memory organization is not a convenience but a computational necessity, and that the advantage grows with the scale of the task.

We test these predictions experimentally on a controlled lookup benchmark using GPT-4o-mini as the primary agent, with replication on GPT-5.4.
Three content types---random key-value hashes, ordered integer lists, and encyclopedia entries---isolate the effect of content familiarity on retrieval behavior.
On abstract content that the model cannot generate from training, indexed retrieval works exactly as predicted: median 1 page read, constant regardless of store size, up to 5,000 items.
On familiar encyclopedia content, a qualitatively different failure mode emerges.
The model recognizes the domain, generates answers from parametric memory instead of reading them from pages, and burns through the token budget in hallucination-driven loops.
This parametric memory competition is content-driven, not structure-driven: the same index architecture that works flawlessly on integers fails catastrophically on encyclopedia entries.

The finding has a concrete architectural implication.
Language models are powerful semantic processors: they should build indices, because understanding content helps decide what to name, where to place, and how to organize.
But they should not traverse indices, because semantic familiarity creates a competing pathway that bypasses the protocol.
Index traversal should be deterministic: spell the filename, get the page, no generation involved.

\subsection{Contributions}

\begin{enumerate}[nosep]
  \item A formal model identifying the transformer context window with the I/O page, yielding an exponential separation between sequential and indexed retrieval (\S\ref{sec:theory}).
  \item Three quantitative predictions---the in-context ceiling, the indexed advantage, and compounding with depth---with specific falsification conditions (\S\ref{sec:predictions}).
  \item A controlled benchmark across three content types and two model generations (GPT-4o-mini and GPT-5.4), confirming the retrieval separation, testing whether sort order substitutes for an explicit index, and revealing parametric memory competition as a failure mode specific to familiar content (\S\ref{sec:experiments}).
  \item A design principle for agent memory systems: semantic models for index construction, deterministic algorithms for index traversal (\S\ref{sec:discussion}).
\end{enumerate}

\section{Preliminaries}\label{sec:prelim}

\subsection{The context window as I/O page}

Let $\mathcal{M}$ denote a transformer-based language model.
At each inference step, $\mathcal{M}$ receives an input of at most $C$ tokens (the context window) and produces an output of at most $C$ tokens.
We treat each inference step as an I/O operation: the model reads one page (the input) and writes one page (the output).

\begin{definition}[Page]\label{def:page}
A \emph{page} is a contiguous block of at most $C$ tokens.
The external store $\mathcal{S}$ consists of $N$ pages $\{p_1, \ldots, p_N\}$.
At each step, the model selects one page to read and produces one page of output.
\end{definition}

This abstraction captures the essential constraint: the model can process at most $C$ tokens per step.
Information not present in the current input is inaccessible regardless of the model's capacity.

\subsection{Access models}

We define two access models for the external store.

\begin{definition}[Sequential access]\label{def:sequential}
Under \emph{sequential access}, the model may read any page by its position index: $\texttt{read}(i)$ returns $p_i$.
No structural information about page contents is available before reading.
\end{definition}

\begin{definition}[Indexed access]\label{def:indexed}
Under \emph{indexed access}, the store is organized as a B-tree of branching factor~$b$.
Internal nodes contain keys that partition the key space; leaves contain data pages.
$\texttt{read\_index}(k)$ returns the internal node covering key~$k$; $\texttt{read}(i)$ returns the data page.
Navigation from root to leaf requires $\lceil \log_b N \rceil$ index reads plus one data read.
\end{definition}

The branching factor $b$ is determined by the page capacity.
Each index entry requires $\eta$ tokens for the key (or filename), $\kappa$ tokens for the separator key that delimits the child's range, and $\delta$ tokens of per-entry formatting overhead (delimiters, line breaks, page numbers).
Thus $b = \lfloor C / (\eta + \kappa + \delta) \rfloor$.
All asymptotic results depend only on $b \geq 2$; the precise value of $b$ affects constants but not the exponential separation.

The sequential model captures flat conversation history, unstructured scratchpads, and any store where the only way to find information is to scan.
The indexed model captures file systems, database indices, and any store where a bounded-size lookup structure directs the model to the correct page.

\section{The Library Theorem}\label{sec:theory}

\begin{theorem}[Sequential retrieval cost]\label{thm:sequential}
Under sequential access, finding a target item in a store of $N$ pages requires $\Omega(N)$ page reads in the worst case.
Under uniform random placement, the expected cost is $(N+1)/2$.
\end{theorem}

\begin{proof}
The model has no information about page contents before reading.
An adversary can place the target on any unread page, forcing at least $N$ reads in the worst case.
Under uniform random placement, the target is equally likely to be on any page; the expected number of reads before finding it is $(N+1)/2$.
\end{proof}

\begin{theorem}[Indexed retrieval cost]\label{thm:indexed}
Under indexed access with a B-tree of branching factor~$b$, finding a target item in a store of $N$ pages requires at most $\lceil \log_b N \rceil + 1$ page reads.
\end{theorem}

\begin{proof}
Each internal node of the B-tree partitions the search space by a factor of~$b$.
Starting from the root, $\lceil \log_b N \rceil$ comparisons identify the correct leaf, which is read in one additional operation.
\end{proof}

\begin{theorem}[Library separation]\label{thm:library}
The separation between sequential and indexed retrieval is exponential:
\[
\frac{R_{\text{seq}}(N)}{R_{\text{idx}}(N)} = \frac{\Omega(N)}{\bigO(\log_b N)} = \Omega\!\left(\frac{N}{\log_b N}\right).
\]
\end{theorem}

\begin{theorem}[Reasoning cost accumulation]\label{thm:reasoning}
Consider a reasoning process of $T$ steps, where at each step the agent may read from or write to the external store.
After $t$ steps, the store contains $N_0 + t$ pages.
\begin{itemize}[nosep]
  \item Under sequential access, the cumulative retrieval cost is $\sum_{t=1}^{T} \Omega(N_0 + t) = \Omega(T \cdot N_0 + T^2/2)$.
  \item Under indexed access, the cumulative cost is $\sum_{t=1}^{T} \bigO(\log_b(N_0 + t)) = \bigO(T \log_b(N_0 + T))$.
\end{itemize}
The separation ratio grows as $\Omega(T / \log_b T)$ for large~$T$.
\end{theorem}

\begin{corollary}[Token cost separation]\label{cor:tokens}
In an agentic loop where each API call receives the full message history, the FLAT agent's $k$-th call processes $\bigO(k)$ prior pages.
Total token cost is $\bigO(N^2)$ for a store of $N$ pages.
The INDEXED agent reads a fixed-size index of $\bigO(N)$ tokens plus one page; total cost is $\bigO(N)$.
The token cost ratio is $\Theta(N)$.
\end{corollary}

\subsection{Interpretation}

The theorems say something precise about the relationship between a bounded processor and an organized versus unorganized store.
The model has no persistent state across steps beyond what is carried in the next input.
Finding the correct answer is therefore equivalent to finding the correct input tokens: the page whose contents, when loaded into the window, enable the desired computation.

This equivalence gives the separation its specific character.
Without a file system, the transformer can only reach adjacent pages, so its reachability grows linearly: $R(k) = k$.
With a B-tree, each page contains $b$ pointers to children, so reachability grows exponentially: $R(k) = b^k$.
The separation is not an artifact of the model; it is a direct consequence of branching.
Linear reachability yields linear search cost; exponential reachability yields logarithmic search cost.

In the dynamic case (Part~II), the store grows by one page per step.
Each new read must search a larger store, so the linear cost accumulates quadratically while the logarithmic cost accumulates log-linearly.
The compounding is not a separate phenomenon but the static separation applied to an expanding domain.
For tasks involving large accumulated state (debugging, research, long-horizon planning), the retrieval component is typically the dominant cost, and the separation applies with full force.

\section{Predictions and Evidence}\label{sec:predictions}

The Library Theorem is a worst-case result about idealized models.
Its value depends on whether the formal separation connects to observable phenomena.
We identify three predictions, each with a specific empirical test and falsification condition.

\subsection{The in-context ceiling (P1)}

\paragraph{Prediction.}
On retrieval-heavy tasks where accumulated state exceeds the context window, no in-context reasoning strategy will systematically outperform CoT self-consistency at matched token budgets.

\paragraph{Formal basis.}
Theorem~\ref{thm:sequential}: all strategies confined to the context window share the sequential reachability bound $R(k) \leq k$.
Strategies may differ in how they allocate tokens within the window (debate, reflection, tree search), but none escapes the linear retrieval floor.

\paragraph{Existing evidence.}
\citet{wang2023selfconsistency} find that self-consistency gains from multiple CoT samples plateau beyond 5--10 samples on arithmetic and commonsense reasoning tasks, consistent with a linear retrieval ceiling.
\citet{snell2024scaling} show that scaling test-time compute via best-of-$n$ sampling yields diminishing returns compared to scaling model parameters, measuring this on GSM8K and MATH.
Neither study was designed to test P1 directly; both measure total task performance rather than isolating retrieval cost.

\paragraph{Falsification.}
P1 would be falsified by an in-context strategy that achieves super-linear reachability, accessing information from more than $k$ pages in $k$ steps without external memory.
A concrete test: on a synthetic retrieval task with $N$ facts, measure accuracy as a function of~$N$ for the best in-context strategy; P1 predicts accuracy degrades as $\Omega(1/N)$.

\paragraph{Boundary.}
On computation-bound tasks where the difficulty is in per-step computation rather than retrieval, all strategies sharing the same computational class will perform comparably regardless of memory organization.
P1 applies specifically where $N \gg C$.

\subsection{The indexed advantage (P2)}

\paragraph{Prediction.}
Among external-memory strategies, structured scratchpads (with headers, named sections, cross-references) will outperform unstructured scratchpads (flat text append) at matched total token budgets.
The advantage should scale as $\Omega(N / \log_b N)$.

\paragraph{Formal basis.}
Theorem~\ref{thm:library}: indexed retrieval costs $\bigO(\log_b N)$ versus $\bigO(N)$ for sequential scan.

\paragraph{Falsification.}
P2 would be falsified if structured scratchpads showed no advantage over flat scratchpads at large~$N$, or if the advantage were constant rather than growing.
The critical test is not whether structured helps at one scale, but whether the advantage \emph{increases} with~$N$.
A flat delta (constant improvement regardless of~$N$) would suggest the gain comes from prompt formatting rather than retrieval efficiency.

\paragraph{Relation to RAG.}
Retrieval-augmented generation \citep{lewis2020rag} implements the static case of Theorem~\ref{thm:library}: pre-existing documents, pre-built index, read-only access.
Its success across knowledge-intensive tasks constitutes indirect evidence for the formal separation.
Theorem~\ref{thm:reasoning} extends the prediction to a setting RAG does not address: indexing self-generated reasoning state.
Current agents index external knowledge via RAG but store their own intermediate results as flat conversation history.

\subsection{Compounding with depth (P3)}

\paragraph{Prediction.}
At greater reasoning depths~$T$, the indexed advantage should compound per Theorem~\ref{thm:reasoning}, with the separation ratio growing as $(N_0 + T) / \log_b(N_0 + T)$.

\paragraph{Falsification.}
P3 would be falsified if the indexed advantage saturated or reversed at greater depths.
Saturation could occur if index maintenance errors accumulated and degraded retrieval accuracy faster than the store grew.
Reversal could occur if the overhead of maintaining the index consumed enough inference steps to offset the retrieval savings.

\paragraph{Empirical test.}
Fix $N_0$ and vary $T$ across 5--7 levels.
Measure total inference cost (or accuracy at fixed cost) with and without indexing.
Plot the separation ratio against~$T$ and fit a functional form.
P3 predicts a monotonically increasing ratio with slope consistent with $(N_0 + T) / \log_b(N_0 + T)$.

\section{Experiments}\label{sec:experiments}

We report controlled experiments testing the predictions of Theorems~\ref{thm:sequential}--\ref{thm:reasoning} against a language model agent.
A single benchmark design is instantiated across three content types to isolate the effect of content familiarity on retrieval behavior.

\paragraph{Notation.}
The theory counts pages ($N$).
The experiments generate $M$ items packed $P$ items per page, yielding $N = \lceil M/P \rceil$ pages.
All experimental parameters are reported in terms of~$M$ (item count); theoretical predictions are translated via $N = M/P$.
The predicted FLAT page-read cost from Theorem~\ref{thm:sequential} is $(N+1)/2 \approx M/(2P)$.
All summary statistics report medians with interquartile ranges (IQR) unless noted otherwise.

\subsection{Benchmark design}\label{sec:design}

A GPT-4o-mini agent must find the value associated with a target key in a paginated store.
The agent issues tool calls---\texttt{read\_page}, \texttt{get\_index}, \texttt{get\_section\_index}, \texttt{submit\_answer}---and cannot inspect any page without explicitly calling the tool.
A token budget of 100,000 per trial aborts runaway hallucination loops.
Trials are seeded by trial index for reproducibility.

Three conditions instantiate the access models of \S\ref{sec:prelim}:

\paragraph{FLAT.}
Pages are ordered randomly with respect to keys.
The agent has only \texttt{read\_page(n)} and \texttt{submit\_answer(v)}.
No structural information is available.
Theorem~\ref{thm:sequential} predicts $\Omega(M/P)$ page reads.

\paragraph{INDEXED.}
Pages are sorted by key.
The agent additionally has \texttt{get\_index()}, which returns a table of contents mapping each page number to its key range.
Theorem~\ref{thm:indexed} predicts $\bigO(1)$ page reads.

\paragraph{FLAT-SORTED.}
Pages are sorted by key in ascending order, and the system prompt tells the agent that pages are sorted.
No index or TOC is provided.
The agent has only \texttt{read\_page(n)} and \texttt{submit\_answer(v)}, but knows that page~1 has the smallest keys and page~$N$ the largest.
This condition tests whether sort order alone---which in principle enables binary search---substitutes for an explicit index.

\paragraph{INDEXED-CORRUPTED.}
Identical to INDEXED but the key-range assignments in the TOC are shuffled: each page number is assigned the key range of a different page.
Following the index leads to the wrong page.
This is a causal control: if performance collapses when the index is corrupted, the INDEXED advantage is causally attributable to the index structure.

\paragraph{DEEP-INDEXED.}
Pages are organized into sections of $S = 10$ pages.
A master TOC lists sections with key ranges; \texttt{get\_section\_index(s)} returns the page-level TOC for section~$s$.
The agent navigates: master TOC $\to$ section TOC $\to$ target page.
Theory predicts $\bigO(1)$ page reads at all~$M$, with three tool calls replacing the single flat TOC.

\subsection{Three content types}\label{sec:content}

The same benchmark architecture is tested with three content types, creating a controlled gradient of content familiarity:

\paragraph{Hash (Experiment~1).}
Keys are random 4-digit integers; values are random 4-letter uppercase strings.
Generated fresh per trial, ruling out memorization.
The model has zero prior knowledge of key-value associations.
This is the purest test of the navigational mechanism.

\paragraph{Numeric (Experiment~2).}
Items are integers $1, \ldots, M$ in sorted order.
Each page shows entries like ``Item 247: 247.''
The content is trivially predictable---Item $k$ has value $k$---but the task still requires the agent to \emph{read} the value from the page and submit it via tool call.
This tests whether the model can follow the retrieval protocol even when it ``knows'' the answer.

\paragraph{Encyclopedia (Experiment~3).}
Keys are English words (alphabetically sorted); values are one-sentence encyclopedia facts.
The corpus is drawn from topics the model has certainly seen during training (``aalii,'' ``abattoir,'' ``acetylene,'' \ldots).
The model has strong parametric knowledge of these topics.
This tests whether content familiarity disrupts the retrieval protocol.

\medskip\noindent
All three experiments use $P = 10$ items per page, $S = 10$ pages per section, and identical tool schemas.
The primary model is GPT-4o-mini; we replicate key conditions on GPT-5.4 in \S\ref{sec:multimodel}.
The only variable within each experiment is the content on the pages.

\subsection{Search property (Theorems~\ref{thm:sequential}--\ref{thm:library})}\label{sec:search}

\begin{table}[t]
\centering
\small
\begin{tabular}{@{}r rrrr rrrr@{}}
\toprule
$M$ & FLAT $\tilde{R}$ & SORT $\tilde{R}$ & IDX $\tilde{R}$ & CORR $\tilde{R}$ & FLAT acc & SORT acc & IDX acc & CORR acc \\
\midrule
50  & 3.0  & 4.0  & 1.0 & 3.0  & 100\% & 100\% & 100\% & 94\% \\
100 & 6.0  & 5.0  & 1.0 & 6.0  & 100\% & 100\% & 100\% & 88\% \\
200 & 9.0  & 7.0  & 1.0 & 11.0 & 98\%  & 100\% & 94\%  & 74\% \\
500 & 21.0 & 21.0 & 1.0 & 15.5 & 94\%  & 97\%  & 94\%  & 58\% \\
\bottomrule
\end{tabular}
\caption{Median page reads ($\tilde{R}$) and accuracy by condition and item count $M$ (hash content, GPT-4o-mini, $P = 10$). FLAT and INDEXED: 50 trials/cell; FLAT-SORTED and CORRUPTED: 30 and 50 trials/cell respectively. Predicted FLAT cost: $M/20 \in \{2.5, 5.0, 10.0, 25.0\}$. SORT = FLAT-SORTED, IDX = INDEXED, CORR = INDEXED-CORRUPTED.}
\label{tab:search}
\end{table}

Table~\ref{tab:search} reports the core results on hash content (Experiment~1).
FLAT tracks the theoretical prediction: the observed medians---3.0, 6.0, 9.0, 21.0---match $M/20$ closely across the full range.
INDEXED holds constant at 1.0 median page reads for all $M \in \{50, 100, 200, 500\}$, confirming the $\bigO(1)$ prediction.

FLAT-SORTED tests whether sort order alone can substitute for an explicit index.
At small $M$ the model achieves a modest advantage (median 4--7 reads versus 3--9 for FLAT), but at $M = 500$ the advantage vanishes: median 21 reads, identical to FLAT.
The model attempts binary search at small scales but cannot sustain it: it loses track of its search bounds, revisits pages, and reverts to near-linear scanning.
Sort order provides the \emph{information} needed for logarithmic retrieval but not the \emph{mechanism}---the model must compute the search strategy on every step, and this computation degrades with scale.

\begin{figure}[t]
\centering
\includegraphics[width=0.95\textwidth]{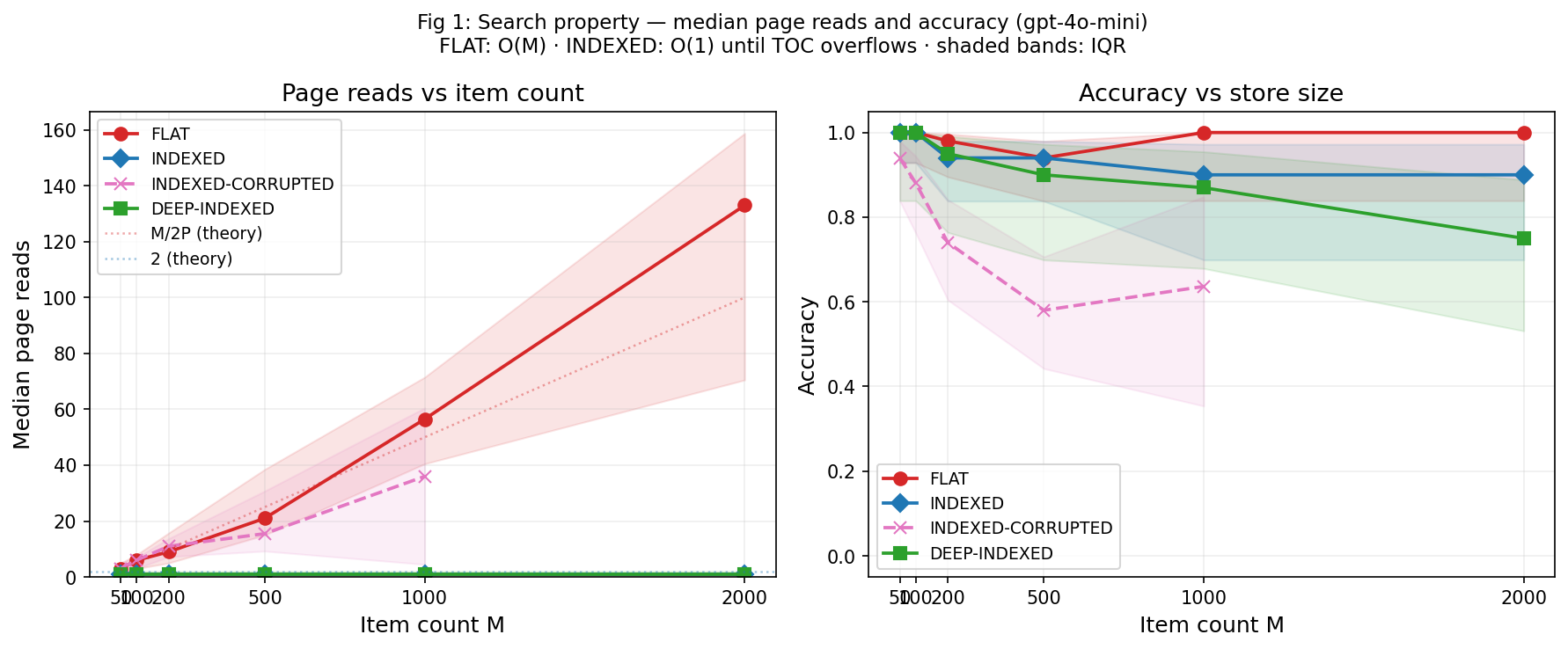}
\caption{Page reads ($R$) and accuracy versus item count~$M$ for the three conditions (hash content).
FLAT (red) grows linearly, tracking the $M/(2P)$ prediction.
INDEXED (blue) remains constant at ${\approx}1$.
INDEXED-CORRUPTED (pink) exceeds FLAT at large~$M$, confirming causal index use.
Shaded bands show interquartile range.}
\label{fig:search}
\end{figure}

The separation ratio (Figure~\ref{fig:separation}, left) grows at $3\times$, $6\times$, $9\times$, $21\times$ for $M \in \{50, 100, 200, 500\}$, tracking the theoretical $M/20$ curve with no free parameters.
Figure~\ref{fig:separation} (right) shows per-trial distributions: FLAT broadens rightward with~$M$ while INDEXED collapses to a point mass at~1.

\begin{figure}[t]
\centering
\includegraphics[width=0.95\textwidth]{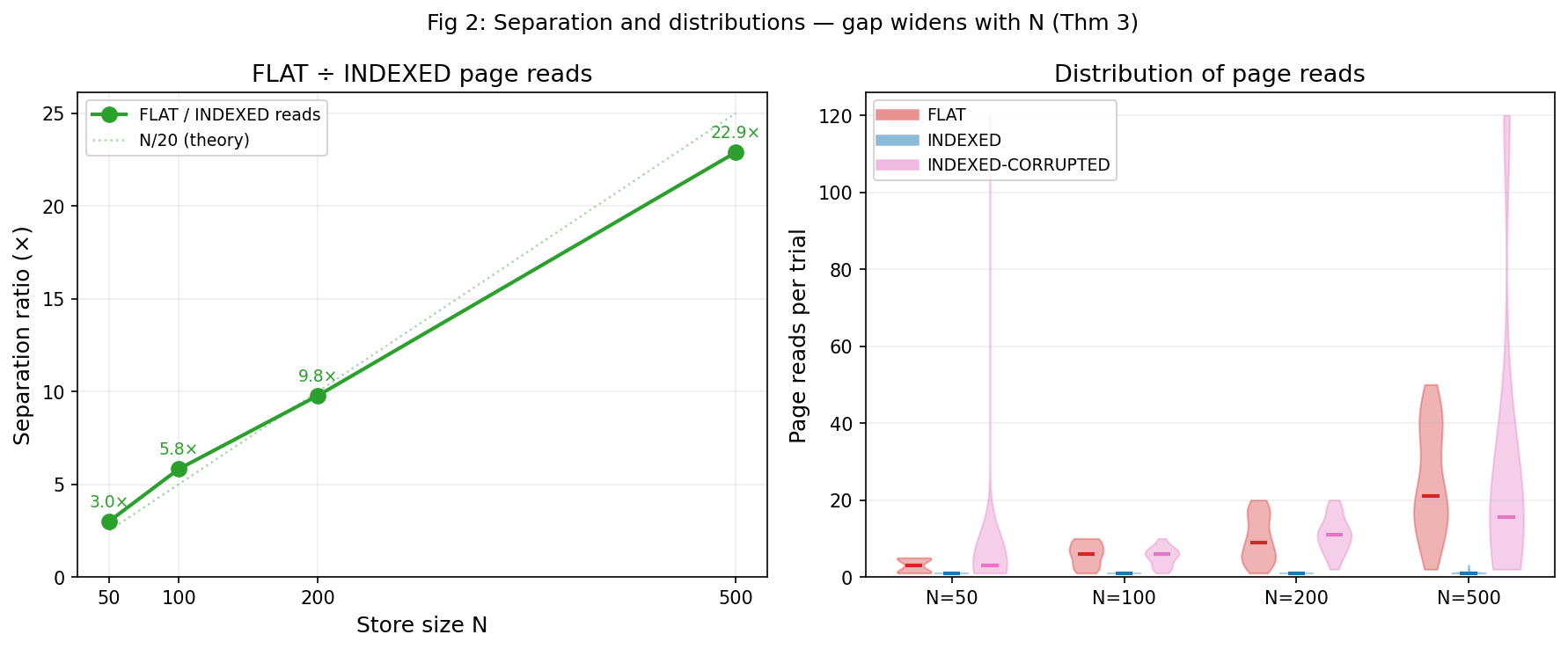}
\caption{Left: separation ratio (FLAT/INDEXED page reads) versus~$M$, with theoretical prediction $M/20$ (dashed).
Right: per-trial read distributions by condition at each~$M$.
FLAT broadens rightward; INDEXED collapses to~1.}
\label{fig:separation}
\end{figure}

The CORRUPTED condition establishes causal attribution.
At $M = 200$, CORRUPTED median reads (11.0) exceed FLAT (9.0); at $M = 500$ the CORRUPTED median is 15.5 versus 21.0 for FLAT, but accuracy collapses to 58\% (versus 94\% for FLAT).
The model follows the corrupted index to the wrong page, fails to find the target, and then falls back to scanning---paying one wasted read on top of the expected FLAT cost.
The model is actively using the index.
When the index is wrong, accuracy degrades below the unindexed baseline.

\subsection{Token cost (Corollary~\ref{cor:tokens})}\label{sec:tokens}

\begin{table}[t]
\centering
\begin{tabular}{@{}r r r r@{}}
\toprule
$M$ & FLAT $\widetilde{\mathrm{Tok}}$ & IDX $\widetilde{\mathrm{Tok}}$ & Ratio \\
\midrule
50    & 1,284    & 842     & 1.5$\times$ \\
100   & 3,191    & 952     & 3.4$\times$ \\
200   & 5,904    & 1,172   & 5.0$\times$ \\
500   & 25,005   & 1,832   & 13.7$\times$ \\
1,000 & 174,264  & 3,348   & 52.0$\times$ \\
2,000 & 913,983  & 5,950   & 153.6$\times$ \\
\bottomrule
\end{tabular}
\caption{Median total tokens ($\widetilde{\mathrm{Tok}}$) per trial (hash content). Large-$M$ cells (1,000 and 2,000) use 20 trials; all others use 50. INDEXED token cost grows linearly (proportional to TOC size); FLAT cost grows quadratically.}
\label{tab:tokens}
\end{table}

\begin{figure}[t]
\centering
\includegraphics[width=0.95\textwidth]{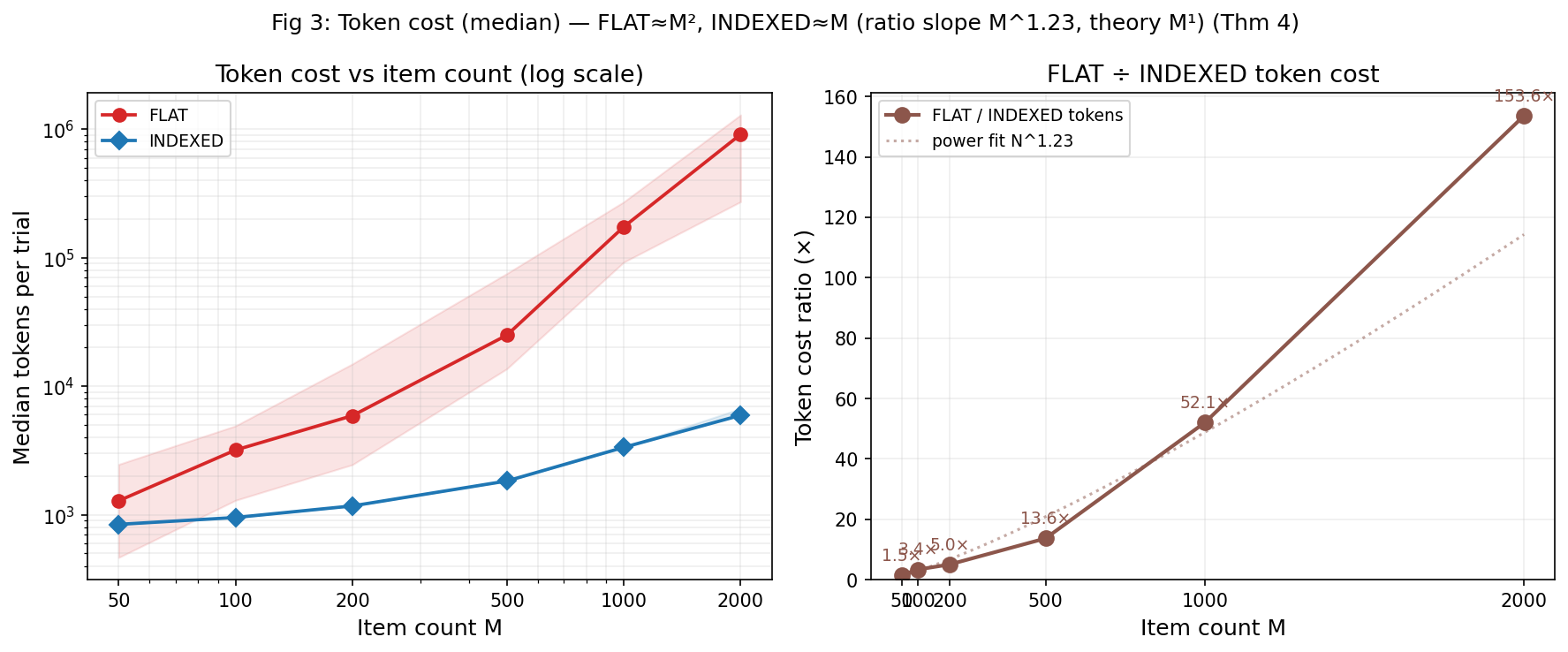}
\caption{Token cost ($\mathrm{Tok}$) versus item count~$M$ (hash content).
Left: median total tokens per trial by condition (shaded bands show interquartile range).
FLAT scales as ${\approx}M^{2}$; INDEXED is orders of magnitude cheaper.
Right: token cost ratio FLAT/INDEXED (median).}
\label{fig:cost}
\end{figure}

FLAT token cost scales quadratically (Table~\ref{tab:tokens}): doubling~$M$ from 1,000 to 2,000 increases the median from 174K to 914K tokens ($5.3\times$, consistent with $M^2$ growth).
The INDEXED cost grows linearly, driven by the growing TOC size.
At $M = 2{,}000$, the FLAT agent consumes 914K tokens per query while INDEXED requires 6K---a $154\times$ separation.

\subsection{Multi-model replication}\label{sec:multimodel}

We replicate the FLAT-SORTED, INDEXED, and DEEP-INDEXED conditions on GPT-5.4 to test whether a substantially more capable model closes the separation gap (Table~\ref{tab:multimodel}, Figure~\ref{fig:multimodel}).

\begin{table}[t]
\centering
\begin{tabular}{@{}l r r r r r r@{}}
\toprule
 & \multicolumn{3}{c}{Median reads $\tilde{R}$} & \multicolumn{3}{c}{Accuracy} \\
\cmidrule(lr){2-4} \cmidrule(lr){5-7}
$M$ & SORT & IDX & DEEP & SORT & IDX & DEEP \\
\midrule
\multicolumn{7}{@{}l}{\emph{GPT-4o-mini}} \\
50  & 4.0  & 1.0 & 1.0 & 100\% & 100\% & 100\% \\
100 & 5.0  & 1.0 & 1.0 & 100\% & 100\% & 100\% \\
200 & 7.0  & 1.0 & 1.0 & 100\% & 94\%  & 95\%  \\
500 & 21.0 & 1.0 & 1.0 & 97\%  & 94\%  & 90\%  \\
\midrule
\multicolumn{7}{@{}l}{\emph{GPT-5.4}} \\
50  & 3.0  & 1.0 & 1.0 & 100\% & 100\% & 100\% \\
100 & 2.0  & 1.0 & 1.0 & 100\% & 100\% & 100\% \\
200 & 3.0  & 1.0 & 1.0 & 100\% & 100\% & 100\% \\
500 & 5.0  & 1.0 & 1.0 & 100\% & 100\% & 100\% \\
\bottomrule
\end{tabular}
\caption{Multi-model comparison on hash content (30 trials/cell). SORT = FLAT-SORTED, IDX = INDEXED, DEEP = DEEP-INDEXED. GPT-5.4 achieves near-optimal binary search on sorted pages (5 reads at $M = 500$ vs.\ theoretical $\log_2 50 \approx 5.6$) but INDEXED still dominates at 1 read.}
\label{tab:multimodel}
\end{table}

\begin{figure}[t]
\centering
\includegraphics[width=0.95\textwidth]{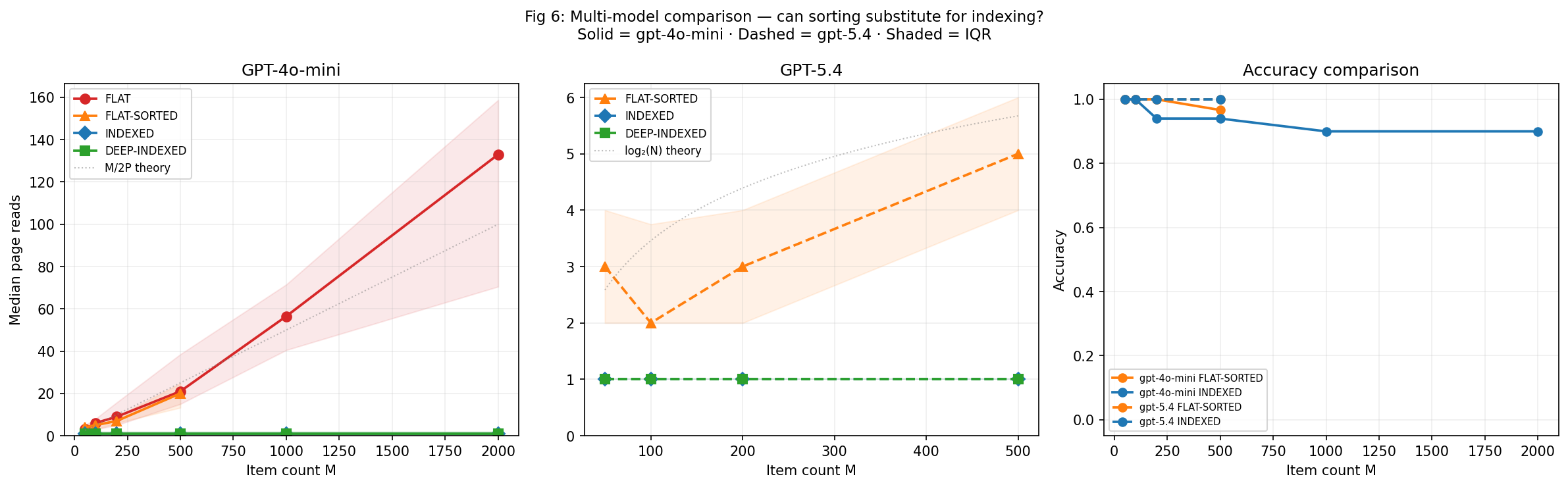}
\caption{Multi-model comparison. Left: GPT-4o-mini cannot sustain binary search on sorted pages; FLAT-SORTED converges to FLAT at $M = 500$.
Center: GPT-5.4 binary-searches effectively ($\tilde{R} = 5$ at $M = 500$), but INDEXED and DEEP-INDEXED hold at~1.
Right: accuracy by condition and model.
Solid lines: GPT-4o-mini. Dashed: GPT-5.4.}
\label{fig:multimodel}
\end{figure}

GPT-5.4 achieves near-optimal binary search on sorted pages: median 2--5 reads across $M \in \{50, \ldots, 500\}$, tracking the $\log_2(M/P)$ curve closely.
At $M = 500$, its 5 reads versus GPT-4o-mini's 21 is a $4.2\times$ improvement---direct evidence that model capability affects the exploitation of structural regularities.
The stronger model is performing \emph{in-context self-indexing}: maintaining search bounds, bisecting, and narrowing, without an external index to guide it.

Yet INDEXED retrieval still dominates: 1 read at all $M$, for both models.
The separation between FLAT-SORTED and INDEXED is $5\times$ even for the stronger model at $M = 500$, and the gap is exponential---at $M = 10{,}000$, binary search would require ${\approx}10$ reads while the index still requires~1.
The index provides a qualitatively different access pattern: the model follows a pointer rather than computing a search strategy.

GPT-5.4 also achieves 100\% accuracy across all conditions and scales tested, compared to GPT-4o-mini's degradation at large $M$.
This confirms that the accuracy shortfalls observed with the weaker model are navigational failures (reaching the wrong page), not limitations of the benchmark design.

\subsection{Two-level indexing (DEEP-INDEXED)}\label{sec:deep}

At $M = 1{,}000$, the flat TOC returned by \texttt{get\_index()} contains $N = 100$ page entries.
Rather than identifying the target range in a single pass, the model searches within the index itself---scanning or binary-searching the 100-entry list.
This is the exact failure mode Theorem~\ref{thm:sequential} predicts when applied to the index as the object of search.
The Library Theorem is recursive: it applies to any sequential store, including the index over a data store.

DEEP-INDEXED restores $\bigO(1)$ retrieval by adding a second level.
Table~\ref{tab:deep} and Figure~\ref{fig:deep} report the results across all three content types.

\begin{table}[t]
\centering
\begin{tabular}{@{}l r r r r r@{}}
\toprule
Content & $M$ & IDX $\tilde{R}$ & DEEP $\tilde{R}$ & IDX acc & DEEP acc \\
\midrule
Hash    & 50    & 1.0 & 1.0 & 100\% & 100\% \\
        & 100   & 1.0 & 1.0 & 100\% & 100\% \\
        & 500   & 1.0 & 1.0 & 94\%  & 90\%  \\
        & 1,000 & 1.0 & 1.0 & 90\%  & 87\%  \\
        & 2,000 & 1.0 & 1.5 & 90\%  & 76\%  \\
        & 5,000 & --- & 2.0 & ---   & 85\%  \\
\midrule
Numeric & 50    & 1.0 & 1.0 & 100\% & 100\% \\
        & 100   & 1.0 & 1.0 & 100\% & 100\% \\
        & 200   & 1.0 & 1.0 & 100\% & 100\% \\
        & 500   & 1.0 & 1.0 & 100\% & 100\% \\
\midrule
Encycl. & 50    & 1.0 & 1.0 & 100\% & 90\%  \\
        & 100   & 1.0 & 1.0 & 100\% & 90\%  \\
        & 200   & 1.0 & 0.0 & 100\% & 27\%  \\
        & 500   & 3.5 & 1.0 & 90\%  & 43\%  \\
\bottomrule
\end{tabular}
\caption{INDEXED versus DEEP-INDEXED across three content types. Median page reads ($\tilde{R}$) and accuracy. Hash and numeric content: retrieval works as predicted. Encyclopedia content: DEEP-INDEXED accuracy collapses at $M \geq 200$ despite median reads near 1, indicating the model generates answers from parametric memory rather than reading them from pages. The median reads of 0.0 at $M = 200$ means the majority of trials hit the token budget before reading \emph{any} page. 30 trials per cell except hash $M \geq 1{,}000$ (50 trials) and hash $M \geq 3{,}000$ (20 trials).}
\label{tab:deep}
\end{table}

\begin{figure}[t]
\centering
\includegraphics[width=0.95\textwidth]{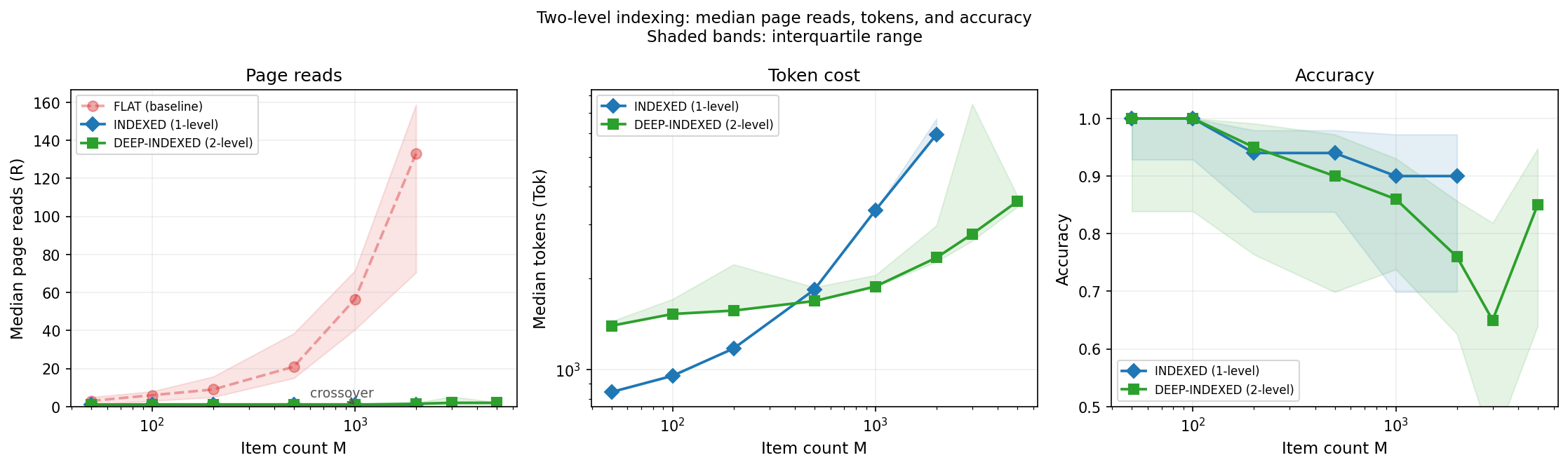}
\caption{Two-level indexing: median page reads, tokens, and accuracy across store sizes.
INDEXED (blue) degrades at $M \geq 1{,}000$; DEEP-INDEXED (green) holds near 1--2.
FLAT (red, dashed) shown for reference.
Shaded bands show interquartile range.}
\label{fig:deep}
\end{figure}

On hash content, DEEP-INDEXED achieves median 1--2 page reads at all scales tested, including $M = 5{,}000$ (500 pages, 50 sections).
At $M = 1{,}000$, where INDEXED degrades, DEEP-INDEXED holds steady, directly confirming the recursive prediction.
The rare accuracy shortfalls at large $M$ (76\% at $M = 2{,}000$) reflect extraction errors---the model reaches the correct page but copies the wrong value---not navigational failures.

On numeric content, the results are even cleaner: 100\% accuracy, exactly 1 page read, zero variance in token counts across all 240 trials.
The model follows the protocol perfectly when the content is abstract and orderly.

\subsection{Parametric memory competition}\label{sec:parametric}

\begin{figure}[t]
\centering
\includegraphics[width=0.95\textwidth]{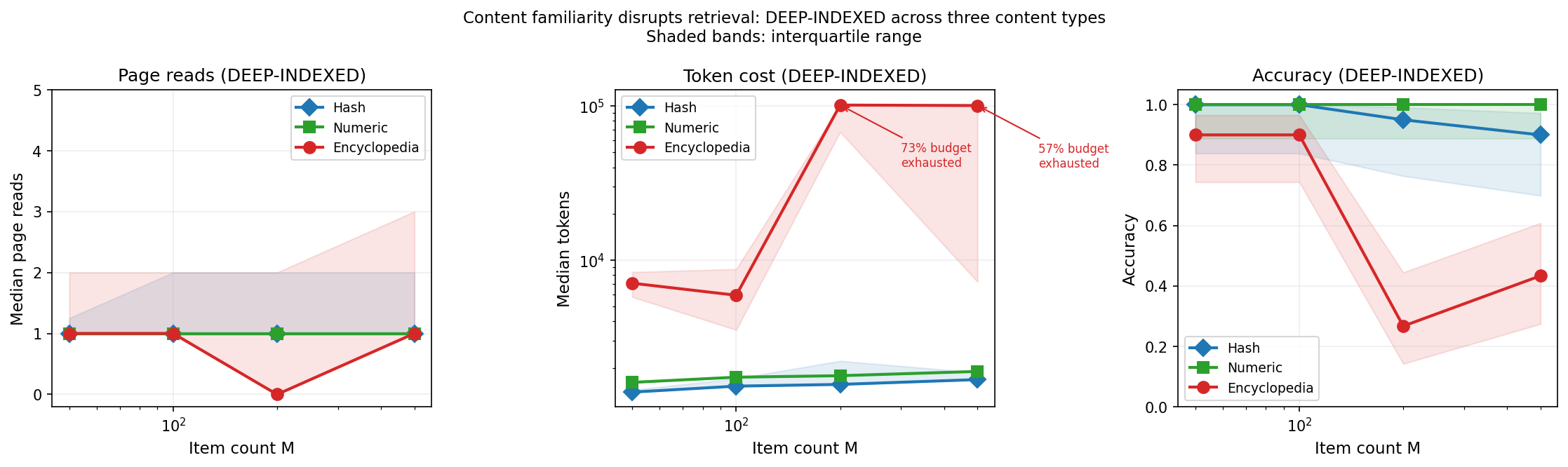}
\caption{Content familiarity disrupts retrieval.
DEEP-INDEXED performance across three content types at matched store sizes.
Left: page reads---hash and numeric hold at~1; encyclopedia collapses.
Center: token cost---encyclopedia inflates by $50\times$ as trials exhaust the 100K budget.
Right: accuracy---encyclopedia drops to 27\% at $M = 200$.
Shaded bands show interquartile range.}
\label{fig:content}
\end{figure}

The encyclopedia results reveal a qualitatively different failure mode.
DEEP-INDEXED accuracy collapses from 90\% at $M \leq 100$ to 27\% at $M = 200$ and 43\% at $M = 500$.
At $M = 200$, 22 of 30 trials (73\%) exhaust the 100K token budget; at $M = 500$, 17 of 30 (57\%) do so.
The median page reads of 0.0 at $M = 200$ means the majority of trials never read \emph{any} data page---the model spends its entire token budget in hallucination-driven tool-call loops, generating plausible answers from parametric memory rather than following the retrieval protocol.

This failure is not structural.
The same DEEP-INDEXED architecture achieves 100\% accuracy on numeric content and 87--100\% on hash content at the same scales.
The index is correctly constructed.
The navigation tools are correctly implemented.
The failure is content-driven: the model recognizes the encyclopedia domain and activates a competing pathway---generating the answer from training rather than reading it from the page.

The three content types form a controlled gradient:

\begin{enumerate}[nosep]
  \item \textbf{Hash} (zero familiarity): random strings the model has never seen.
    The model \emph{must} read the page; there is no parametric alternative.
    Result: protocol followed, retrieval works.
  \item \textbf{Numeric} (predictable but abstract): Item $k$ has value $k$.
    The model could in principle generate the answer, but the content is sufficiently abstract that it follows the tool-call protocol.
    Result: protocol followed perfectly, zero variance.
  \item \textbf{Encyclopedia} (high familiarity): real-world facts about English words.
    The model has strong parametric knowledge and activates it.
    Result: protocol bypassed, catastrophic failure.
\end{enumerate}

The gradient isolates content familiarity as the causal variable.
The index structure is held constant; the tools are held constant; the model is held constant.
Only the content changes, and with it the model's tendency to shortcut the retrieval protocol.

\paragraph{Mechanism.}
The failure pattern is consistent with a \emph{competition} between two inference pathways:
(1)~the \textbf{retrieval pathway}, which follows the tool-call protocol (read index $\to$ read section $\to$ read page $\to$ submit value), and
(2)~the \textbf{parametric pathway}, which generates the answer directly from training data.
When the content is unfamiliar, only the retrieval pathway produces a plausible answer, and the model follows it.
When the content is familiar, the parametric pathway produces a plausible answer \emph{before} the retrieval pathway completes, and the model acts on it---submitting a generated answer, or continuing to generate text instead of issuing tool calls, burning through the token budget.

This is not a failure of the index.
It is a failure of the model to \emph{defer to} the index when it believes it already knows the answer.
The distinction matters for system design: the solution is not a better index but a separation of the two operations.

\paragraph{Token accounting.}
The cost asymmetry is stark.
On numeric content at $M = 500$: DEEP-INDEXED uses median 1,897 tokens, zero variance, 100\% accuracy.
On encyclopedia content at $M = 500$: DEEP-INDEXED uses median 100,597 tokens (IQR: 7,232--102,920), with 57\% of trials hitting the budget ceiling.
The $53\times$ token inflation is entirely attributable to the parametric pathway hijacking the retrieval protocol.

\section{Related Work}\label{sec:related}

\paragraph{Transformer complexity.}
The formal foundations rest on three results.
\citet{feng2023mystery} establish that constant-depth transformers without chain-of-thought compute exactly $\mathrm{TC}^0$.
\citet{merrill2024expressive} show that polynomial-length CoT extends this to $\mathrm{P}$.
\citet{li2025computational} prove the WINDOW = SPACE equivalence that grounds our bottleneck property.
These results characterize \emph{what transformers can compute}.
The Library Theorem asks a different question: \emph{how efficiently can they retrieve?}
\citet{tomlinson2026bapo} and \citet{amiri2025lower} prove CoT step lower bounds (computation-bound floors); the Library Theorem proves retrieval-bound floors.
A complete theory of reasoning cost requires both.

\paragraph{I/O complexity.}
The I/O complexity model of \citet{aggarwal1988input} separates fast internal memory from slow external memory and counts block transfers.
The B-tree of \citet{bayer1972organization} is optimal in this model: $\bigO(\log_B N)$ search in a store of $N$ blocks of size~$B$.
We identify the transformer context window with the I/O page, so that $C$ plays the role of~$B$.
This identification imports B-tree optimality results directly into the analysis of transformer retrieval cost and provides a formal explanation for the empirical advantage of tool-augmented agents, which the expressivity hierarchy alone ($\mathrm{TC}^0 \subset \mathrm{P}$) cannot account for.

\paragraph{Neural external memory.}
The Neural Turing Machine \citep{graves2014neural} and Differentiable Neural Computer \citep{graves2016hybrid} augment neural networks with external memory accessed through learned soft attention.
These systems implement approximate content-based addressing, a fundamentally different mechanism from the file system's exact string-equality lookup.
The Library Theorem's formalization is deliberately simpler: by assuming exact addressing, it isolates retrieval cost from addressing quality.
Whether approximate addressing can achieve B-tree navigation with sufficient reliability is an empirical question that our model does not address.

MemGPT \citep{packer2023memgpt} implements a virtual memory hierarchy for LLMs, managing context overflow with page-in/page-out operations.
This engineering system approximates the indexed access model; Theorem~\ref{thm:reasoning} predicts that deeper hierarchies would yield compounding gains.

\paragraph{Agent theory.}
\citet{schuurmans2023memory} proves that LLMs with unbounded external memory are Turing-complete, establishing what is \emph{possible}.
The Library Theorem adds the efficiency dimension: among Turing-complete systems, the organization of memory determines retrieval cost.
\citet{meyerson2025asymptotic} call for complexity analysis of agent systems treating the inference call as a primitive; we provide this analysis for the retrieval component.
\citet{sumers2024cognitive} propose a cognitive architecture (CoALA) separating working memory from long-term memory; the Library Theorem formalizes the interface, showing that indexed versus flat organization of long-term memory determines whether retrieval is logarithmic or linear.

\citet{kambhampati2024position} argue that LLMs do not reason but serve as components in modular frameworks (LLM-Modulo).
The Library Theorem is compatible with this view.
Even under the interpretation that each inference step is approximate retrieval from a training distribution, the quality of the conditioning context determines output quality, and indexed access to reasoning history provides exponentially richer conditioning than sequential scanning.
The parametric memory competition finding (\S\ref{sec:parametric}) provides direct empirical support for modular decomposition: the model's semantic understanding is valuable for some operations (index construction) but actively harmful for others (index traversal).

\paragraph{Interactive computation.}
\citet{wegner1998interactive} argues that interaction is more expressive than classical algorithmic computation, and \citet{goldin2008interactive} develop the thesis that the Church--Turing barrier falls for persistent interactive systems.
An inscription agent---a transformer that reads from, writes to, and navigates an indexed file system across multiple inference steps---is an instance of a persistent interactive machine.
The Library Theorem adds an efficiency dimension that the interactive computation framework leaves open: among persistent machines with identical computational power, the organization of the persistent store determines retrieval cost, and the separation is exponential.

\paragraph{RAG.}
RAG \citep{lewis2020rag} is the most widely deployed instance of indexed external memory for transformers.
In our framework, RAG implements the static search case: Theorem~\ref{thm:library} applies to a pre-existing corpus with a pre-built index.
Its empirical success across knowledge-intensive tasks constitutes indirect evidence for the formal separation.
\citet{gao2023ragsurvey} survey the resulting literature comprehensively.

The dominant retrieval mechanism in practice differs from the file system model in one important respect.
DPR \citep{karpukhin2020dpr} established the dense retrieval paradigm: queries and passages are encoded by learned bi-encoders, and retrieval is approximate nearest-neighbour search.
\citet{velickovic2025softmax} prove that softmax attention cannot maintain sharp selection as the retrieval database grows, precisely the regime in which the Library Theorem predicts the largest separation.
The Library Theorem's exact-addressing model corresponds to sparse retrieval systems (BM25, structured databases, file systems) and establishes the ceiling that dense approaches approximate.

RETRO \citep{borgeaud2022retro} demonstrates retrieval at the extreme: a 2-trillion-token external datastore, with a 7B-parameter model matching GPT-3 (175B parameters).
This calibrates the scale at which efficient retrieval infrastructure becomes necessary---the regime the Library Theorem characterises.

The Library Theorem predicts a next architectural step that RAG does not address: applying indexed access to self-generated reasoning state (Theorem~\ref{thm:reasoning}).
Self-RAG \citep{asai2023selfrag} is the closest existing system to this dynamic regime.
Theorem~\ref{thm:reasoning} predicts that indexed storage of self-generated state should compound with reasoning depth; Self-RAG's selective retrieval is a step toward this, but without indexed organisation of the agent's own history the quadratic cost remains.

\paragraph{RASP and transduction.}
\citet{weiss2021thinking} and \citet{strobl2024formal} develop the RASP framework for analyzing transformer computations as formal transductions.
A RASP construction for B-tree navigation would strengthen the feasibility claim in Theorem~\ref{thm:indexed}.
We leave this as future work, noting that \citet{bhattamishra2024separations} demonstrate that index-lookup operations are feasible in one-layer transformers of logarithmic width.

\section{Discussion}\label{sec:discussion}

\subsection{The file system as architectural principle}

A file system transforms quadratic-scaling reasoning into log-linear reasoning.
By ``file system'' we mean something precise and minimal: a partial function from bounded-length names to bounded-length contents, supporting exact lookup by string equality.
No embedding similarity, no learned attention over memory, no soft matching.
The transformer spells a filename; the system returns the page.

The simplicity of the mechanism contrasts with the magnitude of the separation.
\citet{bayer1972organization} showed that organizing data into a B-tree reduces search from linear to logarithmic, and every file system and database index in use today descends from this observation.
The Library Theorem identifies that transformer-based agents inhabit exactly this setting: the context window is a page, each inference step is an I/O operation, and the classical separation applies with full force to reasoning cost.

The practical implication is concrete.
Tool-augmented agents already issue \texttt{read}, \texttt{write}, \texttt{ls} commands as part of their reasoning loops.
The theorem says this interaction is not a convenience but the mechanism by which they escape the quadratic cost of scanning flat context.
An agent that stores intermediate results as named files and retrieves them by name is performing B-tree search over its own reasoning history.

\subsection{Capacity versus learning}

Every formal result in transformer complexity theory characterizes the ceiling of what the architecture permits, not the floor of what trained models achieve.
The Library Theorem follows this convention: it establishes that the model class \emph{admits} logarithmic retrieval; whether any particular trained model \emph{realizes} it is empirical.

Two observations narrow the gap.
First, B-tree navigation requires only comparison of a target key against $b$ index entries, selection of the correct child, and output of a filename string.
\citet{bhattamishra2024separations} show that index lookup is feasible in one-layer transformers of logarithmic width.
Second, the gap between capacity and learning is narrower for retrieval than for computation.
A model need not discover the B-tree algorithm from scratch; it need only generate the correct filename in response to a page of index entries, a task closer to instruction following than to algorithmic reasoning.

The harder question, raised by \citet{hahn2024sensitivity}, concerns index \emph{construction} rather than navigation.
Building an effective index over one's own reasoning state requires deciding what to name, how to organize, and when to restructure.
The Library Theorem assumes the index exists and is maintained at $\bigO(\log_b N)$ cost per insertion.
Whether models can learn to construct good indices is important but separate; the parametric memory competition finding suggests, however, that construction is precisely where language models' semantic capabilities should be applied.

\subsection{The construction--traversal separation}\label{sec:separation}

The three-experiment comparison yields a design principle that sharpens the capacity-versus-learning distinction.

Index \emph{construction} requires understanding: deciding that a document about ``acetylene'' belongs between ``acetal'' and ``acid,'' or that a code function should be indexed under both its name and its semantic purpose.
This is exactly what language models do well.
Semantic processing helps build good indices.

Index \emph{traversal} requires discipline: reading the TOC entry, extracting the page number, issuing the tool call, reading the result, submitting the value.
This is a deterministic protocol.
Semantic processing does not help and, as the encyclopedia experiment demonstrates, actively hurts: the model's understanding of the content creates a competing pathway that bypasses the protocol.

The implication is that agent memory systems should separate these two operations:
\begin{itemize}[nosep]
  \item \textbf{Construction}: Use language models to build, organize, and maintain the index. Semantic understanding is an asset.
  \item \textbf{Traversal}: Use deterministic algorithms to navigate the index. Spell the filename, get the page. No generation, no semantic processing, no opportunity for the parametric pathway to compete.
\end{itemize}

This separation mirrors the broader LLM-Modulo pattern \citep{kambhampati2024position}: use the model where its strengths apply (understanding, judgment, organization) and deterministic components where reliability is required (lookup, execution, verification).
The Library Theorem provides the formal motivation: the separation is not just good engineering practice but necessary to capture the full exponential advantage that indexed memory provides.

\subsection{Extended cognition}

The Library Theorem has a natural reading in the framework of extended cognition \citep{clark1998extended}: a transformer augmented with an indexed file system is not a transformer with a tool but a different cognitive system, one whose reasoning capacity is constituted in part by the external memory structure.
The parity principle holds directly---the file system plays the same functional role (enabling retrieval of relevant state) that a larger context window would play, and the theorem quantifies the exponential advantage this coupling provides.

This framing is stronger than the working memory analogy sometimes drawn between context windows and human short-term memory.
The mechanisms differ too deeply for that comparison to carry weight \citep{lake2017building}: human working memory involves attentional gating, rehearsal, and modality-specific subsystems with no transformer counterpart.
What the Library Theorem formalizes is closer to what \citet{goody1977domestication} and \citet{donald1991origins} describe as external symbolic storage---the observation that writing, lists, and indexes do not merely record prior thought but restructure what kinds of reasoning become tractable.
\citet{kirsh1994distinguishing} make the related point that creating an index during problem-solving is an epistemic action: it reduces subsequent computational cost rather than advancing directly toward the goal.

\citet{clark2025extended} extends the original thesis to generative AI, arguing that LLMs are candidates for extended cognitive coupling.
The Library Theorem provides formal content for this argument: the coupling between a bounded-capacity processor and an indexed external store yields an exponential separation that neither component achieves alone.

The parametric memory competition adds nuance to this picture.
The extended cognitive system works only when the coupling is disciplined---when the external store is actually consulted rather than bypassed.
A model that ``already knows'' the answer and ignores the index is not an extended cognitive system but a closed one, reverting to the parametric regime with all its limitations.
The construction--traversal separation (\S\ref{sec:separation}) is, in this framing, a condition on the coupling: the external store must be authoritative for the coupling to constitute genuine extension.

\subsection{Limitations}

Several aspects of the model deserve scrutiny.
The sequential access baseline (Definition~\ref{def:sequential}) is a worst-case characterization; real systems may exploit regularities in how state is ordered.
The B-tree assumes a total order on keys, which suffices for many retrieval tasks but not all; stores with complex relational structure may require richer index types.
The amortized analysis in Theorem~\ref{thm:reasoning} assumes that reads and writes alternate at a steady rate; bursty access patterns may have different cost profiles.

The experiments test two models (GPT-4o-mini and GPT-5.4) on a single task (key-value lookup).
The multi-model replication (\S\ref{sec:multimodel}) confirms that the separation holds across model generations.
However, the parametric memory competition may behave differently across model families or content domains.
GPT-5.4's superior in-context binary search demonstrates that model capability modulates the exploitation of structural regularities, but the exponential gap between sorted and indexed access persists.
The finding identifies a failure mode and its mechanism; the multi-model comparison establishes robustness but not universality.

We have not analyzed hash-based $\bigO(1)$ retrieval, which if feasible within $\mathrm{TC}^0$ would give an even larger separation than the B-tree.
This possibility \emph{strengthens} rather than weakens the indexed advantage; the B-tree analysis provides a lower bound on the gain from indexing.

\subsection{Three levels of external memory}

The results invite a taxonomy of agent architectures by memory organization:
\begin{enumerate}[nosep]
  \item \emph{Context-window only.} The agent reasons within a fixed $C$-token buffer. No external memory; retrieval cost is bounded by window size but so is capacity.
  \item \emph{Flat external memory.} The agent appends to and scans over an external store (conversation history, flat scratchpad, unstructured RAG corpus). Retrieval is $\Omega(N)$ per Theorem~\ref{thm:sequential}; reasoning accumulation is $\Theta(T^2)$.
  \item \emph{Indexed external memory.} The agent navigates a structured store via named pointers (file system, B-tree, database index). Retrieval is $\bigO(\log_b N)$; reasoning accumulation is $\bigO(T \log_b T)$.
\end{enumerate}
The Library Theorem quantifies the gap between levels~2 and~3. Most deployed agents operate at level~2 for self-generated state (flat chat history) while using level~3 for external knowledge (RAG). Theorem~\ref{thm:reasoning} predicts that extending indexed organization to the agent's own reasoning history should yield compounding gains---provided the traversal mechanism avoids the parametric competition trap identified in \S\ref{sec:parametric}.

\subsection{Directions}

The Library Theorem addresses a single agent's interaction with its own memory.
The natural extension is to multiple agents sharing an indexed store, where concurrent reads and writes introduce coordination costs not present in the single-agent case.

The construction--traversal separation suggests a concrete research direction: benchmarking index \emph{construction} by language models.
Can a model, given a stream of documents, build an index that a deterministic traversal algorithm can navigate efficiently?
The Library Theorem assumes the index exists; the quality of model-constructed indices, and the conditions under which construction quality degrades, are empirical questions the current experiments do not address.

More broadly, identifying the transformer context window with the I/O page opens the door to importing a larger body of classical results.
External-memory sorting bounds, cache-oblivious algorithms, and buffer management strategies all have potential analogues in the transformer reasoning setting.
The Library Theorem addresses the simplest case, search and retrieval, but the I/O complexity framework is rich, and the identification may prove fruitful beyond the results presented here.


\section*{Code Availability}
Code for all experiments is available at \url{https://github.com/zmainen/library-theorem}.

\section*{Acknowledgements}
I thank Gonzalo Polavieja for helpful corrections on the manuscript.

\end{document}